\documentclass[11pt]{article}
\usepackage[dvips]{graphicx}
\usepackage{amssymb,amsmath,color}
\usepackage{url}

 \usepackage{graphicx}
\usepackage{stmaryrd}
\usepackage{amsmath}
\usepackage{amsthm}
\usepackage{algorithm}
\usepackage{algpseudocode}
\usepackage{xcolor}

\DeclareMathAlphabet\mathbfcal{OMS}{cmsy}{b}{n}

\usepackage[colorlinks=true, allcolors=blue, pagebackref]{hyperref}       
\renewcommand*{\backrefalt}[4]{%
    \ifcase #1 \footnotesize{(not cited)}%
    \or        \footnotesize{(cited on page~#2)}%
    \else      \footnotesize{(cited on pages~#2)}%
    \fi}

 \RequirePackage{natbib}

\DeclareMathOperator*{\argmax}{arg\,max}

\newtheorem{theorem}{Theorem}[section]
\newtheorem{proposition}{Proposition}[section]

\theoremstyle{definition}
\newtheorem{definition}{Definition}[section]

\theoremstyle{remark}
\newtheorem*{remark}{Remark}

\title{Classifier Calibration with ROC-Regularized Isotonic Regression}

\author{Eug\`ene Berta$^\dagger$, Francis Bach$^\dagger$, Michael Jordan$^{\dagger*}$ \\
$^\dagger$Inria,  Ecole Normale Sup\'erieure, PSL Research University \\
$^*$University of California, Berkeley\\
\url{{eugene.berta,francis.bach,michael.jordan}@inria.fr}}

\date{\today}

\begin{document}
\maketitle

\begin{abstract}

    Calibration of machine learning classifiers is necessary to obtain reliable and interpretable predictions, bridging the gap between model confidence and actual probabilities. One prominent technique, isotonic regression (IR), aims at calibrating binary classifiers by minimizing the cross entropy on a calibration set via monotone transformations. IR acts as an adaptive binning procedure, which allows achieving a calibration error of zero, but leaves open the issue of the effect on performance. In this paper, we first prove that IR preserves the convex hull of the ROC curve---an essential performance metric for binary classifiers. This ensures that a classifier is calibrated while controlling for overfitting of the calibration set. We then present a novel generalization of isotonic regression to accommodate classifiers with $K$ classes. Our method constructs a multidimensional adaptive binning scheme on the probability simplex, again achieving a multi-class calibration error equal to zero. We regularize this algorithm by imposing a form of monotony that preserves the $K$-dimensional ROC surface of the classifier. We show empirically that this general monotony criterion is effective in striking a balance between reducing cross entropy loss and avoiding overfitting of the calibration set.
\end{abstract}

\section{INTRODUCTION}
 
Calibration is a natural requirement for probabilistic predictions. It aligns the outputs of a classifier with true probabilities, according with the intuition that the predictions of our models should match observed frequencies. Several papers have demonstrated empirically that simple machine learning classifiers can exhibit poor calibration, even on very simple datasets \citep{zadrozny_learning_2001, zadrozny_transforming_2002, niculescu-mizil_predicting_2005}. More recently \cite{guo_calibration_2017} showed that deep neural networks suffer from the same problem, due to their tendency to overfit the training data, reviving the community's interest in calibration.

The interpretation of the predictions of machine learning classifiers as probabilities is not possible without calibration. Calibration is desirable in that it provides a lingua franca for multiple users to assess the outputs of a learning system.  It also permits the use of learning systems as modules in complex prediction pipelines---a single module can be updated independently of others if its outputs can be assumed to be calibrated.

\subsection{Calibration}

We let $\mathcal{X}$ and $\mathcal{Y}$ denote the \textit{feature space} and the \textit{output space} of a numerical classification problem, respectively, with $\mathcal{Y} = \{0, 1\}$ in the binary classification setting and $\mathcal{Y} = \{ 1, \dots, K \}$ in the general $K$-class classification setting. We consider a probability distribution for a random variable $(X,Y) \in \mathcal{X} \times \mathcal{Y}$, and a probabilistic classifier $f:\mathcal{X} \rightarrow \mathcal{P}$ making predictions $p = f(x)$ in the \textit{prediction space}~$\mathcal{P}$. In the binary case we take $\mathcal{P}=[0,1]$ and in the multi-class case $\mathcal{P} = \Delta_K$, with $\Delta_K$ the $K$-dimensional simplex $\{p \in \mathbb{R}_+^K | \sum_{i=1}^K p_i = 1 \}$.

\begin{definition}[Calibration, \citealp{foster_vohra,zadrozny_transforming_2002}]\label{eq:Calibration}
    A binary classifier $f:\mathcal{X} \rightarrow [0,1]$ is said to be \textit{calibrated}   if $\mathbb{P}[Y=1|f(X)] = f(X)$, or equivalently $\mathbb{E}[Y|f(X)] = f(X)$.
    For a multi-class classifier $f:\mathcal{X} \rightarrow \Delta_K$, the definition is $\mathbb{E}[Y|f(X)] = f(X)$.
\end{definition}

The concept of calibration has been useful in a variety of applied contexts, notably including weather forecasting~\citep{CalibrationWeather1977}.

\textbf{Evaluating calibration.} We define a criterion that assesses the calibration of a classifier.
\begin{definition}[Calibration error]\label{def:CalibrationError}
    For a classifier $f$, the calibration error is $$\mathcal{K}(f) = \mathbb{E}\big[|\mathbb{E}[Y|f(X)] - f(X)|\big].$$
\end{definition}
This error is usually referred to as the \emph{expected calibration error} (ECE)~\citep{naeini_obtaining_2015, guo_calibration_2017}.

For a discrete set of observed data points, $(x_i, y_i)_{1 \leq i \leq n}$, if the classifier $f$ takes continuous values, the expectation $\mathbb{E}[Y|f(X)]$ needs to be estimated. If the predictions live on a discrete grid $\mathcal{P} = [\lambda_1, \dots, \lambda_m]$, we can readily approximate this expectation. For any index $i$, we have $f(x_i) = \lambda_j$ for some $\lambda_j$ in the grid. We can use all the points for which the prediction was $\lambda_j$ ($S_j = \{k \in \llbracket 1, n \rrbracket \, | \, f(x_k) = \lambda_j \}$) to compute the empirical expectation:
$$ 
    \mathbb{E}[y_i|f(x_i)] \simeq \frac{1}{\#S_j} \sum_{k \in S_j} y_k  .
$$
Plugging in such estimates the calibration error can be approximated. Predictions living on discrete grids have been ubiquitous in the early literature on calibration. In particular, in weather forecasting, the predictions usually live on the grid $[0\%, 10\%, \dots, 100\%]$. In the continuous case of machine learning classifiers, however, it is not clear that such discretizations make sense; in particular, it is not clear how they interact with performance.

\textbf{Calibration and model performance.} Calibration has a long history in the economics and statistical literatures \citep[see][for a recent treatment]{foster_forecast_2021}. A central result is that one can always produce a calibrated sequence of predictions, even if the outcomes are generated by an adversarial player. This surprising result is a consequence of the minimax theorem \citep{hart_calibrated_2023}, and it leads to simple strategies to generate a sequence of forecasts that is asymptotically calibrated against any possible sequence of outcomes. This can be viewed as a positive result, but it also has a negative aspect.  Let us envisage a city where it rains every other day. Predicting a $50\%$ chance of precipitation every day is enough to achieve calibration even if this forecast is quite poor. This suggests that while calibration is useful, it should be considered in the overall context of the accuracy of the forecasts~\citep{foster_calibeating_2022}.

\textbf{Calibration and proper scoring rules.} \cite{brocker_reliability_2009} proved that any proper score can be decomposed into the calibration error and a second \textit{refinement} term. In particular, for the cross entropy loss:
\begin{equation}\label{eq:CrossEntropyDecomposition}
    H(Y, f(X)) = \mathbb{E}[KL(f(X) || \mathbb{P}(Y|f(X))]
    + \mathbb{E}[H(\mathbb{P}(Y|f(X)))] ,
\end{equation} with $H(.,.)$ the cross entropy and $H(.)$ the entropy. Here, we see that the calibration error is expressed in terms of the Kullback-Leibler divergence ($KL$); other criteria can arise depending on the specific proper scoring rule that is chosen.
This confirms that a zero calibration error does not necessarily guarantee good forecasts. Indeed, calibration can be achieved independently of the performance of the classifier. The intuition is that aligning model confidence with probabilities can be done whatever the performance of the model, and the lower the model's accuracy, the less confident it should be in its predictions. Machine learning classifiers are usually able to generate forecasts with good accuracy, but these forecasts are generally not calibrated. The decomposition above shows that calibrating our classifiers might help in reducing the cross entropy loss even further.

\subsection{Calibrating Machine Learning Classifiers}
The machine learning literature has generally employed the following simple data-splitting heuristic to calibrate classifiers. Given $n$ i.i.d data points $(x_i, y_i)_{1 \leq i \leq n} \in (\mathcal{X}, \mathcal{Y})$, a portion of this available data is reserved for calibration (\textit{calibration set}) and the classifier is trained on the rest of the data (\textit{training set}). After the classifier is trained, the held-out calibration set is used to evaluate and correct its calibration error. This paradigm separates the calibration procedure from model fitting, resulting in calibration methods that can be applied to any model. However, holding out a portion of the data for calibration can be problematic in data-sparse applications. Moreover, in the context of online learning, every update to the model requires running the calibration step again. New data points will either be used to improve the model performance (training set) or reduce the calibration error (calibration set). In these cases we see that the data-splitting paradigm sets up a trade-off between calibration and performance.

In addition, calibration procedures that use data splitting rely on the assumption that the data are identically distributed across the calibration set and the test set. The idea is that the calibration error observed on the calibration set can be used to evaluate and correct the calibration error on the underlying data distribution, thus calibrating the model for any point sampled from this distribution.

\textbf{Continuous calibration error.} Let $(x_i, y_i)_{1 \leq i \leq n}$ denote the held-out calibration set. We first evaluate the predictions of the model $f$ on this set: $(p_i = f(x_i))_{1 \leq i \leq n}$. For a standard machine learning classifier, these predictions do not live on a fixed grid; instead, they can take arbitrary values in $[0,1]$ (in the binary case). We remember that the calibration error is intractable in this case.
What is usually done in the literature to overcome this difficulty is to discretize the predictions $(p_i)_{1 \leq i \leq n}$ using a regular binning scheme: $(B_j)_{1 \leq j \leq m} = \{[0, \frac{1}{m}], \dots,[\frac{m-1}{m}, 1]\}$ \citep[see, e.g.,][]{naeini_obtaining_2015, guo_calibration_2017}. The discretized predictions are $\Tilde{p_i} = b_j$, with $b_j$ the center of bin $B_j$ such that the initial prediction $p_i \in B_j$. With these discrete forecasts, an estimate of the calibration error can be computed. However, discretizing has some important drawbacks.  In particular, it is not robust to distributions of scores $f(X)$ that are highly skewed on $[0,1]$, a behavior we often observe in practice. Recent work has tried to come up with more suitable ways to evaluate and visualize calibration error in the case of continuous forecasts \citep{vaicenavicius_evaluating_2019}.

\textbf{Nonparametric model calibration.} In an early paper on calibration for machine learning models, \cite{zadrozny_learning_2001} introduced the method we discussed above---using a fixed binning scheme to discretize the outputs of any probabilistic classifier---in the context of various calibration schemes. They note in particular that it is easy to correct the prediction of the model on each bin by replacing it with the actual observed frequency of outcomes on the calibration set. Under the \textit{i.i.d.}\ assumption, this method is trivially calibrated. It adapts very poorly, however, to skewed distributions of the forecasts, and while achieving calibration it can be very detrimental to the performance of the model.  This led to the development of adaptive binning methods that preserve the calibration guarantees of regular binning while trying to set bin boundaries that are less detrimental to performance.  In particular, 
isotonic regression was employed for adaptive binning by \cite{zadrozny_transforming_2002}, and Bayesian binning schemes have also been proposed~\citep{naeini_obtaining_2015}.

\textbf{Parametric model calibration.} On the other end of the spectrum, a rich literature has arisen using parametric procedures to correct calibration errors. For example, Platt scaling \citep{platt_probabilistic_2000} consists in fitting a sigmoid to the forecasts of the classifier on the calibration set to minimize the cross entropy with the calibration labels. Further developments in the parametric vein include the beta calibration method~\citep{kull_beyond_2017}. Unlike binning methods, these methods have the appeal of learning continuous calibration functions, but they provide no guarantees on calibration. With continuous methods, the calibration error can only be estimated with discretization, which is very limiting. On the other hand, the calibration function lives in a restricted class of functions that is characterized by shape constraints, which yields a regularization prior that mitigates performance degradation arising from overfitting the calibration set.

\section{BINARY CALIBRATION WITH ISOTONIC REGRESSION}

The previous section raises the question of whether it is possible to achieve calibration guarantees while preserving the performance of the initial classifier. The decomposition of proper scoring rules in (\ref{eq:CrossEntropyDecomposition}) suggests that setting the calibration error to zero can improve the cross entropy of the classifier. We will see that isotonic regression actually achieves this twofold objective in the setting of binary classification.

\subsection{Isotonic Regression}
\textbf{Isotonic regression} (see \citealp{Robertson1988OrderRS} for a complete treatment) was first proposed as a nonparametric method to calibrate the probabilities of a binary classifier by \cite{zadrozny_transforming_2002}.

\begin{definition}[Isotonic regression]\label{def:binaryIsotonicRegression}
    Let $n \in \mathbb{N}_+^*$, $(p_i, y_i)_{1 \leq i \leq n} \in (\mathbb{R}^2)^n$ and $(w_i)_{1 \leq i \leq n} \in (\mathbb{R}_+)^n$ a set of positive weights. Assuming the indices are chosen such that $p_1 \leq p_2 \leq \cdots \leq p_n$, isotonic regression solves
    \begin{equation*}
        \min_{r \in \mathbb{R}^n} \frac{1}{n} \sum_{i=1}^n w_i(y_i - r_i)^2 \text{ such that } r_1 \leq r_2 \leq \cdots \leq r_n  ,
    \end{equation*}
    where $r$ can be viewed as a $n$-dimensional vector or a function from $\mathcal{P} = \mathbb{R}$ to $\mathcal{Y} = \mathbb{R}$ with $r(p_i) = r_i$.
\end{definition}

This corresponds to finding the increasing (isotonic) function $r$ of inputs $(p_i)_{1 \leq i \leq n}$ that minimizes the squared error with respect to the labels $(y_i)_{1 \leq i \leq n}$, under a certain weighting $(w_i)_{1 \leq i \leq n}$ of each data sample $(p_i, y_i)_{1 \leq i \leq n}$.
\vskip6pt
\begin{remark}
    The problem established by Definition \ref{def:binaryIsotonicRegression} is a convex optimization problem.
\end{remark}
\begin{remark}
    \cite{Robertson1988OrderRS} (Theorem 1.5.1) showed that IR minimizes any Bregman loss function, in particular, the $KL$ divergence. In the framework of supervised-learning, where the target distribution $y$ is fixed, $KL$ is equal to cross entropy up to a constant factor, so IR minimizes the cross entropy loss.
\end{remark}

\textbf{Pool adjacent violators algorithm (PAV)}. The solution of the isotonic regression (IR) problem can be found via the acclaimed PAV algorithm \citep{Ayer1955ANED}. This algorithm is a very simple procedure (see Algorithm~\ref{alg:PAV}) that has $O(n)$ computational complexity. A proof that PAV solves the IR problem can be found in \citet{Robertson1988OrderRS}. 

\begin{algorithm}
\caption{Pool Adjacent Violators}\label{alg:PAV}
\begin{algorithmic}
\Require $p_1 \leq p_2 \leq \cdots \leq p_n$
\State $\forall i \in \llbracket1,n\rrbracket, r_i \gets y_i$
\While{not $r_1 \leq r_2 \leq \cdots \leq r_n$} \Comment{Until $r$ is monotone}
    \If{$r_i < r_{i-1}$} \Comment{Find adjacent violators}
        \State $r_i \gets \frac{w_i r_i + w_{i-1} r_{i-1}}{w_i + w_{i-1}}$ \Comment{Pool}
        \State $w_i \gets w_i + w_{i-1}$ \Comment{Pool}
        \State Remove $r_{i-1}$ and $w_{i-1}$ from the list. \Comment{Pool}
    \EndIf
\EndWhile
\end{algorithmic}
\end{algorithm}

\subsection{Isotonic Regression is Calibrated}
In practice, we use our classifier $f$ to generate non-calibrated forecasts on the calibration set $(p_i = f(x_i))_{1 \leq i \leq n}$. We then fit IR with these non-calibrated forecasts in input and calibration labels $(y_i)_{1 \leq i \leq n}$ as targets with constant weights $\forall i, w_i = 1$. This gives us a new set of calibrated forecasts $(r_i)_{1 \leq i \leq n}$.

When IR was introduced in the context of probability calibration \citep{zadrozny_transforming_2002}, it was presented as an alternative to binning and Platt scaling. We see from Algorithm \ref{alg:PAV} that IR produces a piece-wise constant function. Moreover, on each constant region the value of the function is the mean of the labels $y_i$ for all $p_i$ falling in this region. Theses two simple observations show that IR produces an \textit{adaptive binning scheme} for which the bin boundaries are set so that the resulting function is increasing. This binning-like property allows us to recover interesting guarantees from the nonparametric calibration methods that we presented earlier.

\begin{proposition}
    The isotonic regression $(r_i)_{1 \leq i \leq n}$ of one-dimensional inputs $(p_i)_{1 \leq i \leq n} \in \mathbb{R}$ to binary labels $(y_i)_{1 \leq i \leq n} \in \{0,1\}$ achieves zero calibration error, that is, $\mathcal{K}(r, y) = 0$.
\end{proposition}
\begin{proof}
    The value of $r$ at any point can be written:
    $$ 
    r(p) = \frac{1}{\#\{p_i \in B_j\}}\sum_{p_i \in B_j} y_i  ,
    $$
    for some bin $B_j$ in a finite set of bins $(B_j)_{1 \leq j \leq m}$, such that $p \in B_j$.
    Moreover, $r$ is increasing and takes only $m$ distinct values $[b_1, \dots, b_m]$. For any $p \in \mathbb{R}$, the events $\{p \in B_j\}$ and $\{r(p) = b_j\}$ are equivalent. Thus,
    \begin{align*}
        \mathbb{E}[Y|r(p) = b_j] &=  \frac{1}{\#\{r(p_i) = b_j\}}\sum_{r(p_i) = b_j} y_i \\[-.025cm]
                                  &=   \frac{1}{\#\{p_i \in B_j\}}\sum_{p_i \in B_j} y_i .
    \end{align*}
    So, $\forall p \in \mathbb{R}, \mathbb{E}[Y|r(p)] - r(p) = 0$, and the calibration error is zero.
\end{proof}

This proof formalizes the idea that generalized binning schemes provide calibration guarantees and it applies for any binning scheme in an input space of any dimension.

Considering $r$ as a piece-wise constant function, we obtain a mapping that we can apply to any future forecast to correct the inherent mis-calibration bias of our initial classifier. Under the assumption that the data are i.i.d across the test set and calibration set, we can thus bound the calibration error on the test data \citep[cf.][]{zhang_risk_2002}. 

\subsection{Isotonic Regression Preserves ROC-AUC}
As discussed in the context of evaluating calibration error, a large binning scheme makes coarse approximations of the original function which might result in less accurate predictions. On the other hand, a thin binning scheme can approximate well the initial function but it reduces the number of points per bin and it can lead to overfitting of the calibration set (it also reduces the calibration guarantee that we obtain). We thus obtain a trade-off between overfitting the calibration set and sacrificing initial model performance. Given that IR behaves as an adaptive binning scheme, let us explore how it performs vis-a-vis this trade-off.

One essential assumption that we make with isotonic regression is that the calibration function $f$ is increasing. Taking $(p_i)_{1 \leq i \leq n}$ to be the outputs of our original binary classifier and the resulting $(r_i)_{1 \leq i \leq n}$ to be the calibrated version of these probabilities, this implies that $(r_i)_{1 \leq i \leq n}$ preserves the ordering of $(p_i)_{1 \leq i \leq n}$. Thus, under this assumption, we obtain a first guarantee that isotonic regression preserves the quality of the original predictions.

However, we only enforce $r_i \leq r_{i+1}$ and not $r_i < r_{i+1}$. The ordering is only partially preserved as we can set consecutive $p_i \neq p_{i+1}$ to take the same value $r_i = r_{i+1}$. The PAV algorithm starts with the perfect fit, nonincreasing in general, such that $ r_i = y_i, \forall i \in \llbracket 1, n \rrbracket$. It then merges consecutive values where the current approximation of the target function is decreasing, $r_{i+1} < r_i$, which means that the original ordering of $p_i$ and $p_{i+1}$ was wrong. Setting $r_{i+1} = r_i$ in this case actually corresponds to solving an ordering issue of the original sequence and might well improve the quality of our predictions. To formalize this simple intuition, we need the following definition:

\begin{definition}[Symmetric ROC curve]
    The simplex $\Delta_2$ can be reduced to the $[0,1]$ interval on $\mathbb{R}$. For different values of threshold $\gamma \in [0,1]$, we can split the simplex in two parts $R_0 = [0, \gamma]$ and $R_1 = \: ]\gamma, 1]$ and evaluate $p_0(\gamma) = \mathbb{P}(X \in R_0 | Y\!=\!0)$, $p_1(\gamma) = \mathbb{P}(X \in R_1 | Y\!=\!1)$. We define the symmetric ROC curve (SROC) as the two-dimensional graph
    $
    \big\{ \big( p_0(\gamma), p_1(\gamma) \big), \gamma \in \mathbb{R} \big\}  .
    $
\end{definition}

\begin{remark}
    The symmetric ROC curve is exactly the classical ROC curve up to an inversion of the $x$-axis \citep{fawcett_roc_2006}. Our definition exposes a symmetry that will lead to a natural generalization in the next section. The area under the ROC curve (AUC) is the same under the two conventions.
\end{remark}

\cite{Provost_Robust_2001} and \cite{bach_considering_2006} described how one can convexify the ROC curve of a classifier by taking convex combinations of decision rules corresponding to different thresholds $\gamma$ (in particular, averaging between the points forming the convex hull of the ROC curve).  Moreover, they showed that the convex hull of the ROC curve is a more robust performance criterion than the initial ROC curve.

\begin{theorem}
    The ROC curve of isotonic regression is the convex hull of the ROC curve of the initial classifier.
\end{theorem}

\begin{proof}
    IR finds the left derivative of the greatest convex minorant (GCM) of the cumulative sum diagram (CSD)~\citep[][Theorem 1.2.1]{Robertson1988OrderRS}:
    $$
        \big\{ \big( \sum_{i=1}^j w_i, \sum_{i=1}^j w_i y_i \big),  j \in \llbracket 1, n \rrbracket \big\}  .
    $$ Thus, IR has a convex CSD that is the GCM of the original CSD.
    This property is illustrated with a simple example in Figure \ref{fig:IRROCConvex}. PAV has a natural interpretation as an iterative procedure to build the GCM of a discrete graph. In terms of cumulative probabilities, the CSD can be interpreted as:
    $$ 
        \big\{ \big( \mathbb{P}(X \leq p_j), \mathbb{P}(X \leq p_j \cap Y=1)  \big),  j \in \llbracket 1, n \rrbracket \big\} .
    $$
    By a simple affine transformation of the axes, $a_1 = \frac{a_1 - a_2}{\mathbb{P}(Y=0)}$ and $a_2 = 1 - \frac{a_2}{\mathbb{P}(Y=1)}$, we recognize the SROC graph:
    $$ 
        \big\{ \big( \mathbb{P}(X \leq p_j | Y = 0), \mathbb{P}(X \geq p_j | Y = 1)  \big),  j \in \llbracket 1, n \rrbracket \big\}  .
    $$
    This graph re-writing preserves convex sets, so the ROC curve of IR is the convex hull of the ROC curve of the initial classifier, as illustrated in Figure \ref{fig:IRROCConvex}.
\end{proof}

\begin{figure}[h]
    \centering
    \includegraphics[scale=0.48]{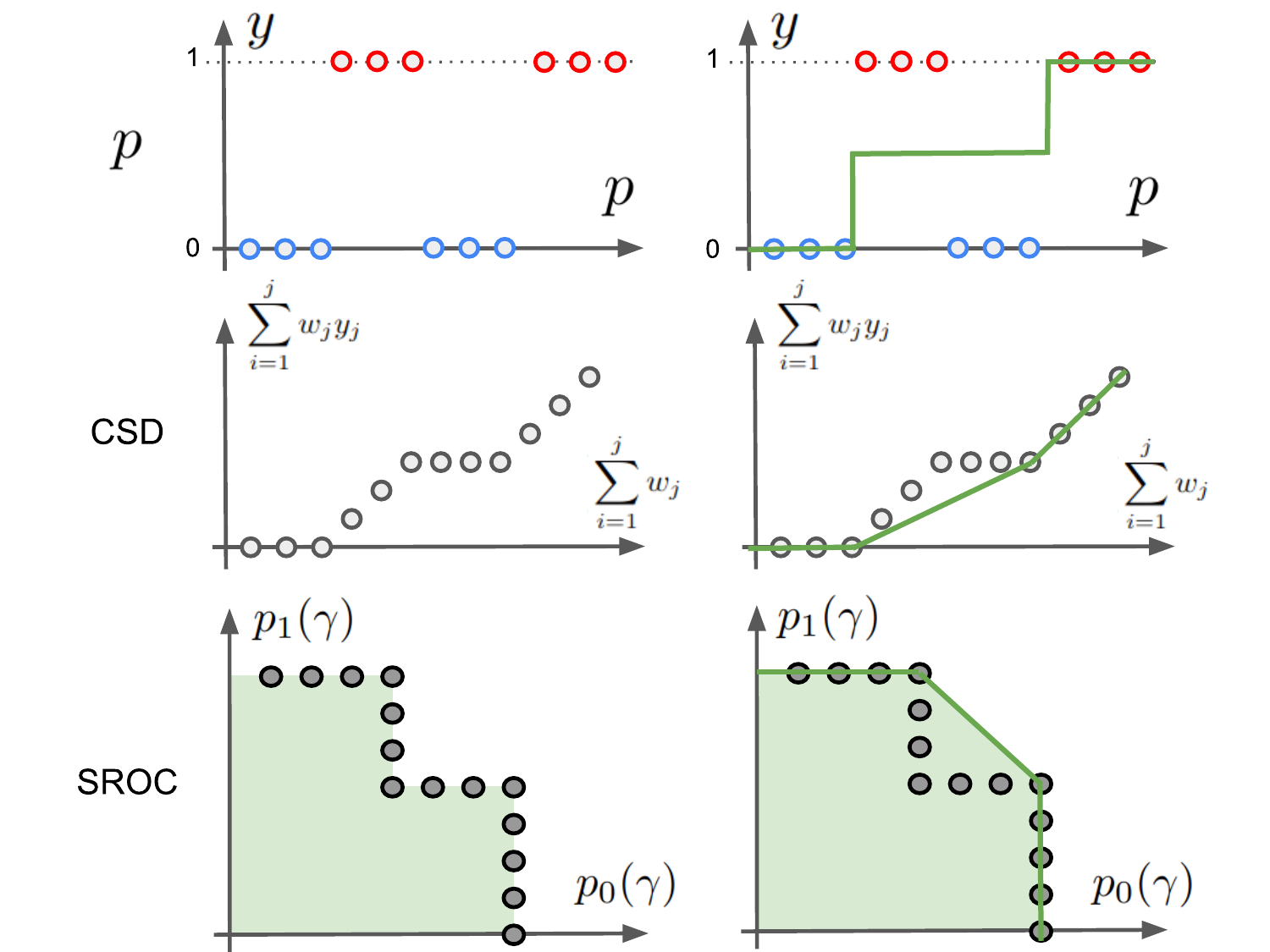}
    \vspace*{-.1cm}
    \caption{Illustrative problem with points spread across two classes \textit{blue} ($y=0$) and \textit{red} ($y=1$). \textbf{Left:} model predictions, CSD, SROC curve. \textbf{Right:} IR (equal to left derivative of the GCM), GCM of the CSD, SROC curve of IR (equal to the convex hull of the initial SROC curve).}
    \label{fig:IRROCConvex}
\end{figure}

A link between IR and the ROC convex hull algorithm was noted previously by \cite{fawcett_pav_2007}. To the best of our knowledge, our proof is the first that establishes this link formally.

IR minimizes the cross entropy on the calibration set but the monotony assumption acts as a regularizer that prevents the calibration function from improving performance further beyond the convex hull of the initial ROC curve. This regularization achieves an optimal trade-off by guaranteeing that we are not hurting performance of the initial model (the AUC is improved or preserved) and prevents overfitting of the calibration set. To illustrates this trade-off, we fit a logistic regression on the first two classes of the Covertype dataset \citep{covertype} and we calibrate our classifier with IR and a recursive binning scheme that makes no monotony assumption.
We fit IR using isotonic recursive partitioning (IRP) \citep{IRP, GIRP}, a recursive procedure that creates new regions in an iterative manner. We plot the cross entropy on the calibration set and on the test set depending on the number of bins created; see Figure \ref{fig:2DIRPvsBinning}. We see that unlike the standard binning procedure that overfits the calibration set when the grid gets too fine, the monotony regularization of IR prevents overfitting, and the algorithm stops when the cross entropy is minimized on the test set. Moreover, the extra freedom that IR can set adaptive bin boundaries results in lower cross entropy with fewer bins than for the standard binning procedure.

\begin{remark}
Standard IR on binary labels starts with a 0-valued bin and ends with a 1-valued bin which can cause the test cross entropy to be infinite in case of misclassification. We regularize IRP by adding Laplace smoothing when computing the means on each bin. This new regularized mean minimizes an entropy regularized cross entropy $H(p, y) - \lambda log(p)$ for some regularization strength $\lambda$ depending on the amount of Laplace smoothing. On the calibration set, we plot that regularized cross entropy, which is minimized by our algorithm. On the test set however, we plot the standard cross entropy.
\end{remark}

\begin{figure}
    \centering
    \includegraphics[scale=0.49]{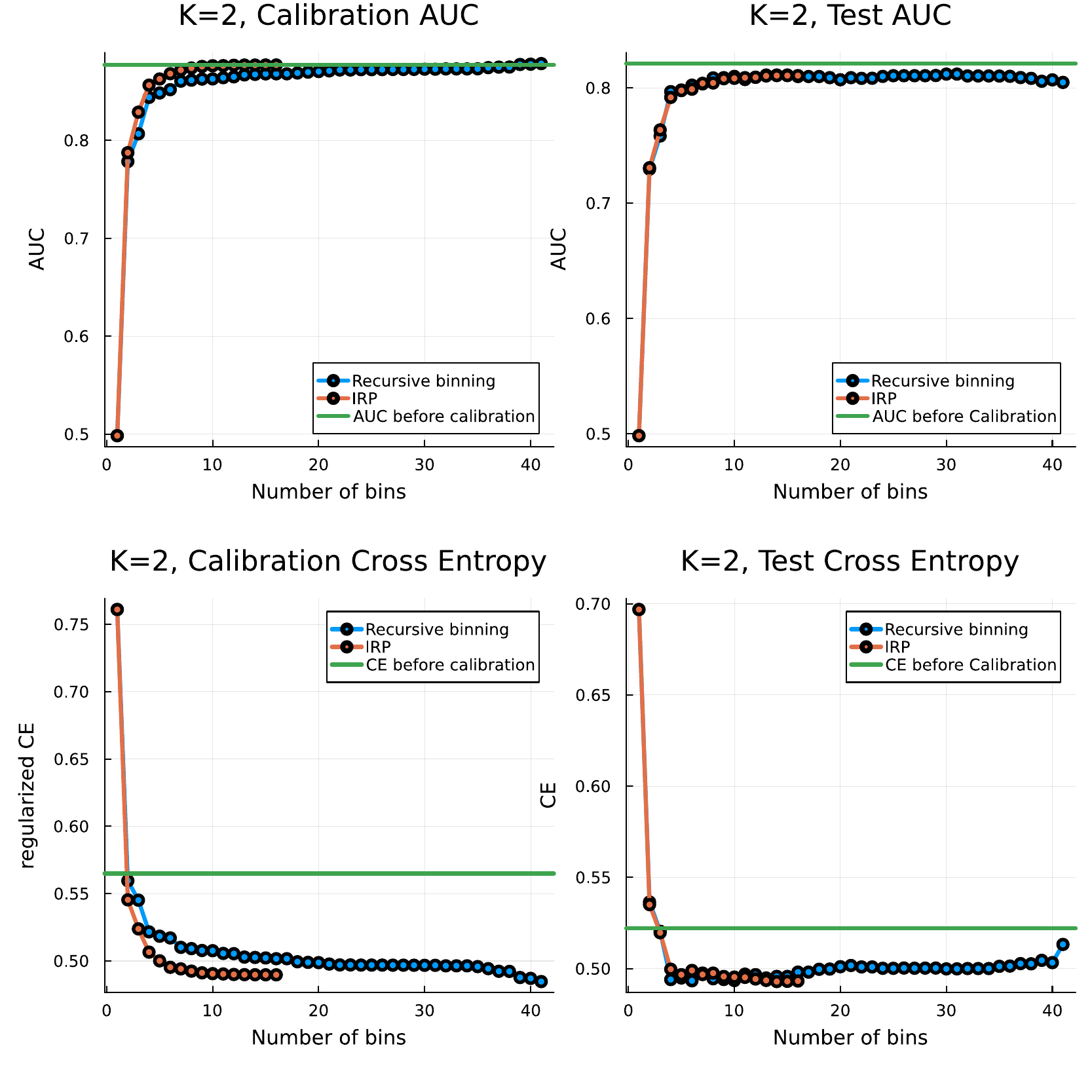}
    \vspace*{-.25cm}
    \caption{Calibration and test cross entropy and AUC, IRP versus nonmonotone recursive binning.}
    \label{fig:2DIRPvsBinning}
\end{figure}

\section{MULTI-CLASS IR}

The previous section presented some of the appealing properties of IR calibration in the binary setting. We now investigate the possibility of building a similar tool for the more general multi-class calibration setting. The definition we use for multi-class calibration requires that predictions are calibrated on every class. This definition is overly restrictive for problems with a large number of classes (typically $K > 5$), for which it is natural in practice to ask that the model is calibrated only on the top classes. For simplicity, we simply focus on low-dimensional classifiers in this paper and leave extensions to high-dimensional classifiers for future work.

Let $K \in \mathbb{N}, K \geq 3$. In the general $K$-class setting, we have $\mathcal{P} = \Delta_K$ and $\mathcal{Y} = \{0, 1, \cdots, K\}$. For convenience, we use the one-hot encoding of the labels $\mathcal{Y} = \Delta_K$.

\subsection{Multi-Class ROC Surface}
In the binary case, our increasing function naturally preserves the ordering of the initial forecasts, which leads us to conclude that it preserves the ROC curve of the initial classifier. In the multi-class setting, a similar notion of ordering is harder to define. Many definitions of multidimensional monotony exist and behave as different regularization hypothesis for our calibration function. To mimic the binary case, we are interested in preserving the ROC curve of the non-calibrated forecasts on the calibration set. To carry out this programme, we first require a definition of the ROC curve in any dimension.

Let $A_K = \{ x \in \mathbb{R}^K | \sum_{k=1}^K x_k = 1 \}$ denote an affine combination of the unit vectors in $\mathbb{R}^K$, and let $\gamma \in A_K$ denote a multi-dimensional threshold. In a similar fashion to the binary case, we can split $\Delta_K$ into $K$ regions, $R_1, R_2, \dots, R_K$, around $\gamma$ and define $K$ probabilities $p_1(\gamma) = \mathbb{P}(X \in R_1 | Y=1), \dots, p_K(\gamma) = \mathbb{P}(X \in R_K | Y=K)$. Varying $\gamma$ allows us to build a $K$-dimensional ROC surface. For a given $\gamma \in A_3$, Figure \ref{fig:splitSimplex} illustrates a natural symmetric splitting of the simplex~$\Delta_3$.

\begin{figure}[h]
    \centering
    \includegraphics[scale=0.25]{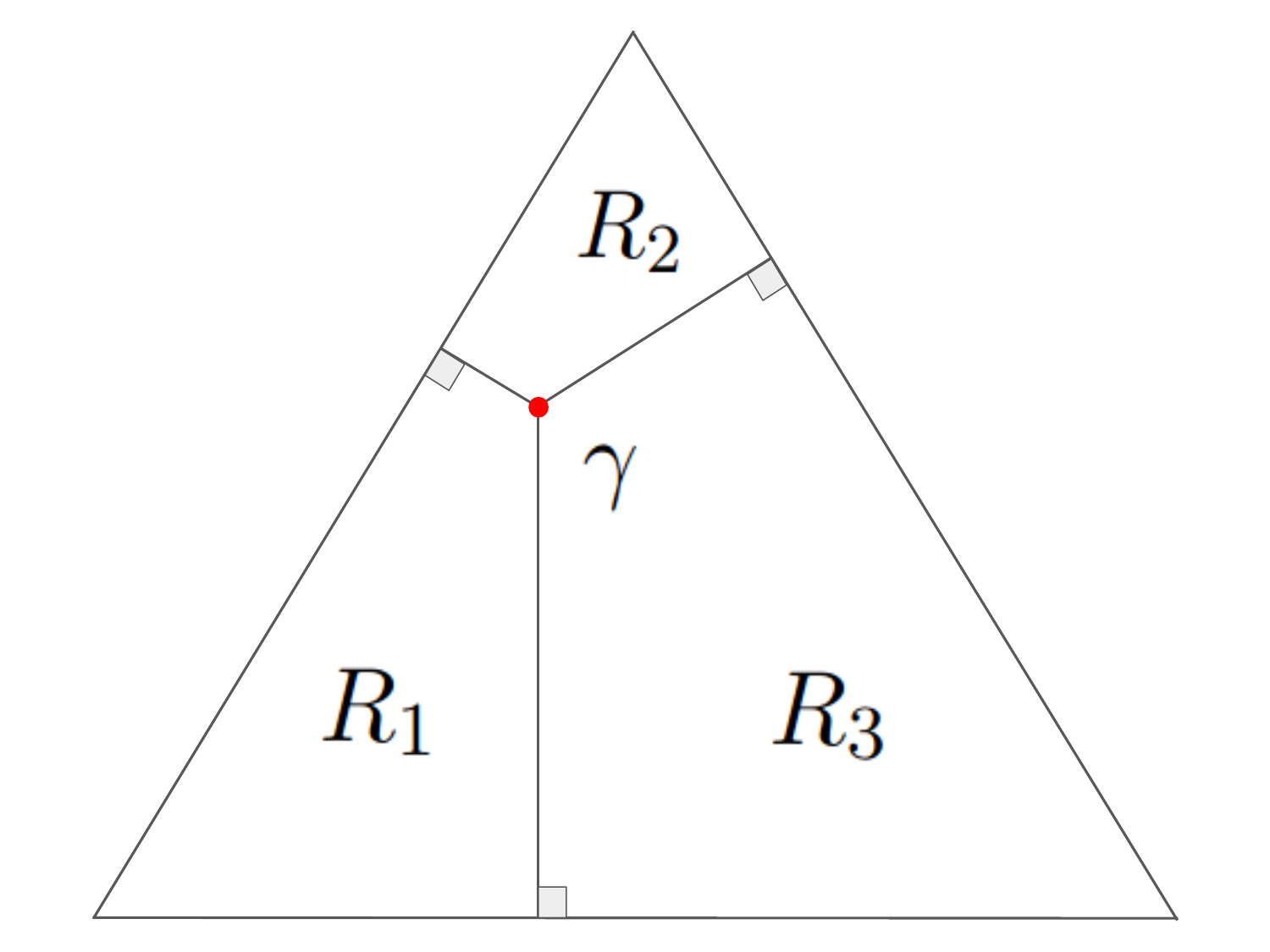}
    \vspace*{-.1cm}
    \caption{Natural splitting of the simplex $\Delta_3$ into class-specific regions $R_1$, $R_2$, $R_3$.}
    \label{fig:splitSimplex}
\end{figure}

This splitting strategy can be extended to build partitions of the simplex around any point $\gamma \in A_K$ in dimension $K$:
\begin{equation}\label{eq:SimplexSplitting}
    R_k = \{ r \in \Delta_K | \argmax_{\llbracket 1, K \rrbracket}(r-\gamma) = k \}  ,
\end{equation}
for all $k \in \llbracket 1, K \rrbracket$.
For any point $r \in \Delta_K$ and $\gamma \in A_K$, the vector $r-\gamma$ is necessarily associated with a maximum-valued axis $k$ such that $r_k - \gamma_k \geq r_i - \gamma_i$, for all $i \in \llbracket 1, K \rrbracket$. The boundaries correspond to ties in the argmax, and the ties can be broken with any strategy that ensures that each point belongs to only one region, such that (\ref{eq:SimplexSplitting}) defines a partition of the simplex.

We also define the subset $S_k$ of points $p$ that belong to region $R_k$ for a given split $\gamma$: $S_k(p, \gamma) = \{p_i \in R_k(\gamma)\}$.

Equipped with this partition of the simplex, we extend the standard definition of the ROC curve to an arbitrary dimension.

\begin{definition}[ROC surface] For a random experiment with outputs $Y \in \Delta_K$, we define the ROC surface of forecasts $P \in \Delta_K$ as the $K$-dimensional graph:
$$
    \big\{ \big( p_1(\gamma), p_2(\gamma), \dots, p_K(\gamma) \big) , \forall \gamma \in A_K \big\}  ,
$$ where $ p_k(\gamma) = \mathbb{P}\big(P \in R_k(\gamma) | Y=k \big)$, for all $k \in \llbracket 1, K \rrbracket$, and $R_k(\gamma)$ was defined above.
\end{definition}

\begin{remark}
    A technical subtlety is that we are using $\gamma \in A_K$ and not $\gamma \in \Delta_K$. In the binary case, taking $\gamma \in \Delta_2$ is enough to build the full ROC curve but this is not true in general. The splitting point must be allowed to take values in the affine plane outside the simplex. Without this additional freedom, for $K=3$ for example it would not be possible to put all the points in the same region, and the points $(0,0,1), (0,1,0), (1,0,0)$ would not belong to the ROC surface.
\end{remark}

This ROC surface illustrates how well our classifier can separate the $K$ classes in the data for any choice of multi-dimensional threshold $\gamma$. The volume under the ROC surface (VUS) can be computed in any dimension to provide an indication of the performance of a multi-class classifier.

\subsection{Generalized Monotony}

This extension of the ROC curve to arbitrary dimensions allows us to define a new monotony criterion that aims at preserving the ROC surface of the initial model. We seek to define constraints on the values of our multidimensional calibration function so that the ROC surface of the calibrated forecasts $r$ is the same as the ROC surface of non-calibrated forecasts $p$. In the binary case, each possible threshold $\gamma \in [0,1]$ generates a split between points $S_0(r, \gamma)$ and $S_1(r, \gamma)$. The fact that the function is monotone guarantees that the same partition of the samples can be found with another split on the non-calibrated forecasts. That is, for all $ \gamma \in [0,1]$, there exists $\gamma' \in [0,1]$ such that $(S_0(p, \gamma'), S_1(p, \gamma')) = (S_0(r, \gamma), S_1(r, \gamma))$, with $\gamma \neq \gamma'$.
\vskip6pt
\begin{remark}
    This property is not reciprocal as IR is not strictly monotone. IR merges values of consecutive points together, deleting a possible split in the calibrated function. This removes a point from the ROC curve, which explains that the ROC curve after calibration contains fewer points than the ROC curve before calibration. IR is optimal as it keeps only the points that form the convex hull of the ROC curve.
\end{remark}

In a similar fashion, we want the splits that we can make on our calibration function to exist also in the non-calibrated forecasts. In other words, the points that we allow on the calibrated ROC surface are the points from the non-calibrated ROC surface.

\begin{definition}[ROC monotony]
    Let $p = (p_i)_{i \in \llbracket1, n\rrbracket}$ denote non-calibrated forecasts and $r = (r_i)_{i \in \llbracket1, n\rrbracket}$ the image of these forecasts through our calibration function. Our function is said to be \textit{ROC monotone} if
    $$
    \forall \gamma \in A_K, \exists \gamma' \in A_K | \: S_k(r, \gamma) = S_k(p, \gamma'), \forall k \in \llbracket 1, K \rrbracket  .
    $$
\end{definition}
As for the binary case we will average labels on bins, which will delete many points from our initial ROC surface. Many of theses points are sub-optimal (not on the ROC convex hull), so our method should choose to preserve optimal points to preserve the convex hull of the initial ROC surface.

\subsection{Recursive Splitting Algorithm}

We need to split the $K$-dimensional simplex into a finite set of bins to guarantee calibration. On each of these bins, the value of our calibration function will be the mean label for the samples of the calibration set that fall into the bin. A simple idea is to start with a constant function on the simplex and recursively split it into smaller regions. Every time we make a new split, we recompute the value of our function on the newly defined regions by taking the mean of the labels from the calibration set for the points that fall in each of these regions. This procedures guarantees that our function stays calibrated.

We also need to enforce our ROC monotony criterion. Every time we make a new split on the simplex, we can make sure that our function is still monotone, and otherwise reject the split. ROC monotony gives us a natural way to split the simplex, recursively employing the orthogonal split  that we defined earlier in (\ref{eq:SimplexSplitting}). After a split, we only need to check the label's means in the $K$ new regions to make sure that the function is still ROC monotone. The algorithm we just described is very similar to IRP, that solves IR in the binary case. We thus adopt the same splitting strategy as in the standard IRP. Given a region $R$ we select the optimal splitting point $\gamma \in R$ by solving:
$$ 
M_R(\gamma) = \max_{\gamma \in R} \sum_{k=1}^K \#S_k(\gamma) |\Bar{y}_R - \Bar{y}_{R_k(\gamma)}| ,
$$ with $\Bar{y}_B$ the mean label for samples falling in bin $B$.

The algorithm converges when it finds no split that leaves the function ROC monotone in any region. At each iteration, we split the region with the largest $M_R(\gamma)$. The resulting Algorithm \ref{alg:GIRP} works in any dimension. For $K=2$ it coincides with IRP and solves IR. For $K\geq3$ it builds a multi-dimensional adaptive ROC preserving binning scheme. To our knowledge, this is the first method that provides multi-class calibration guarantees without resorting to regular binning schemes.

\begin{algorithm}
\caption{multi-class IRP}\label{alg:GIRP}
\begin{algorithmic}
\Procedure{\textbf{split}}{$R, p, r, y$}
    \State $\textit{splitfound} \gets \textbf{False}$
    \State $M \gets 0$
    \For{$\gamma \in R$}
        \State $\forall k, R_k \gets R_k(\gamma)$ \Comment{Compute split}
        \State $\forall k, S_k \gets S_k(\gamma)$ \Comment{Compute split}
        \State $\forall k, \forall p_i \in S_k, \Hat{r}_i = \Bar{y}_{S_k}$ \Comment{Compute split}
        \If{$\Hat{r}$ \textit{ROC monotone} and $M(\gamma) > M$}
            \State $r \gets \Hat{r}$ \Comment{Update function}
            \State $M \gets M(\gamma)$ \Comment{Update max}
            \State $\textit{splitfound} \gets \textbf{True}$ \Comment{Update status}
        \EndIf
    \EndFor
\EndProcedure
\\
\State $r \gets y$ \Comment{Initialize calibration function}
\State $\textit{regions} \gets [\Delta_K]$ \Comment{Initialize regions list}
\While{$\#\textit{regions} > 0$} \Comment{Recursive splitting}
    \State $\textit{bestsplit} \gets \argmax_\textit{regions}(M)$
    \State $R \gets \textbf{popat}(\textit{regions}, \textit{bestsplit})$
    \State $\textit{splitfound}, \hat{r}, R_1, \dots, R_K \gets \textbf{split}(R, p, r, y)$
    \If{$\textit{splitfound}$}
        \State $r \gets \hat{r}$ \Comment{Update calibration function}
        \State $\textit{regions} \gets \textbf{push}(\textit{regions}, [R_1, \dots, R_K])$
    \EndIf
\EndWhile
\end{algorithmic}
\end{algorithm}

\begin{remark}
    In practice, we evaluate ROC monotony only on the splitting points we introduced and not on the full simplex. This means that all the splits we create correspond to points from the initial ROC surface. Artifacts of the multidimensional space make full ROC monotony too restrictive for any split to exist.
\end{remark}
\begin{remark}
    The original IRP can be solved exactly, with the optimal partition of a region found by solving a linear program. We run our algorithm by choosing splitting points on a grid.
\end{remark}
\begin{remark}
    As in the binary case, we use Laplace smoothing when computing the region means.
\end{remark}

The result of our algorithm is illustrated for $K=3$ and $K=4$ in Figure \ref{fig:3DCalibrationFunction} and Figure \ref{fig:4DCalibrationFunction} in the appendix. In Figure \ref{fig:recursiveSplittingROCGraph} we plot the non-calibrated and calibrated ROC surfaces obtained for the three-class problem. As expected, the surface of our calibrated function contains far fewer points that the initial ROC surface, but these points belong to the initial ROC surface. Our algorithm seems to make our calibration function optimal in the sens that our calibrated ROC surface covers the initial ROC surface.

On the three and four, respectively, top classes of the Covertype UCI dataset \citep{covertype}, we fit a logistic regression classifier that we calibrate with multi-class IRP and a non-regularized recursive binning scheme. Figure \ref{fig:3DIRPvsBinning} and Figure \ref{fig:4DIRPvsBinning} show that, as in the binary case, IRP finds a sweet spot between overfitting the calibration set and sacrificing model performance. Our monotony criterion guarantees that the calibration VUS is majorized by the initial VUS of our classifier. Unlike the binary case, our calibration function does not necessarily reach that upper bound. Still, we see empirically that our adaptive binning outperforms regular binning in terms of bin efficiency. Moreover, as in the binary case, our algorithm naturally stops when the test cross entropy is minimized. This illustrates the efficiency of our multi-class ROC monotony regularization.

\begin{figure}[H]
    \centering
    \includegraphics[scale=0.48]{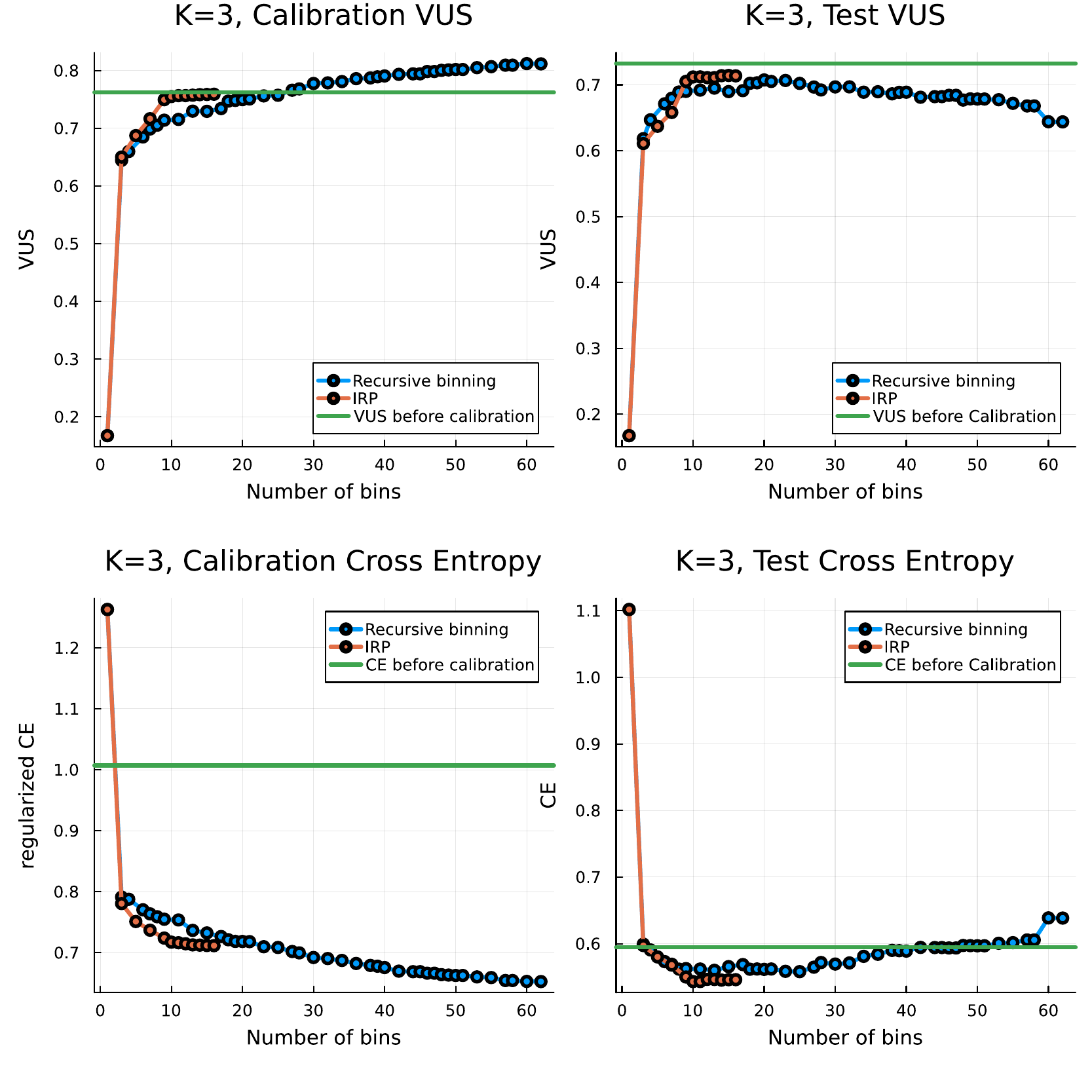}
    \vspace*{-.25cm}
    \caption{For $K=3$, calibration and test cross entropy and VUS, IRP versus nonmonotone recursive binning.}
    \label{fig:3DIRPvsBinning}
\end{figure}

\begin{figure}[H]
    \centering
    \includegraphics[scale=0.48]{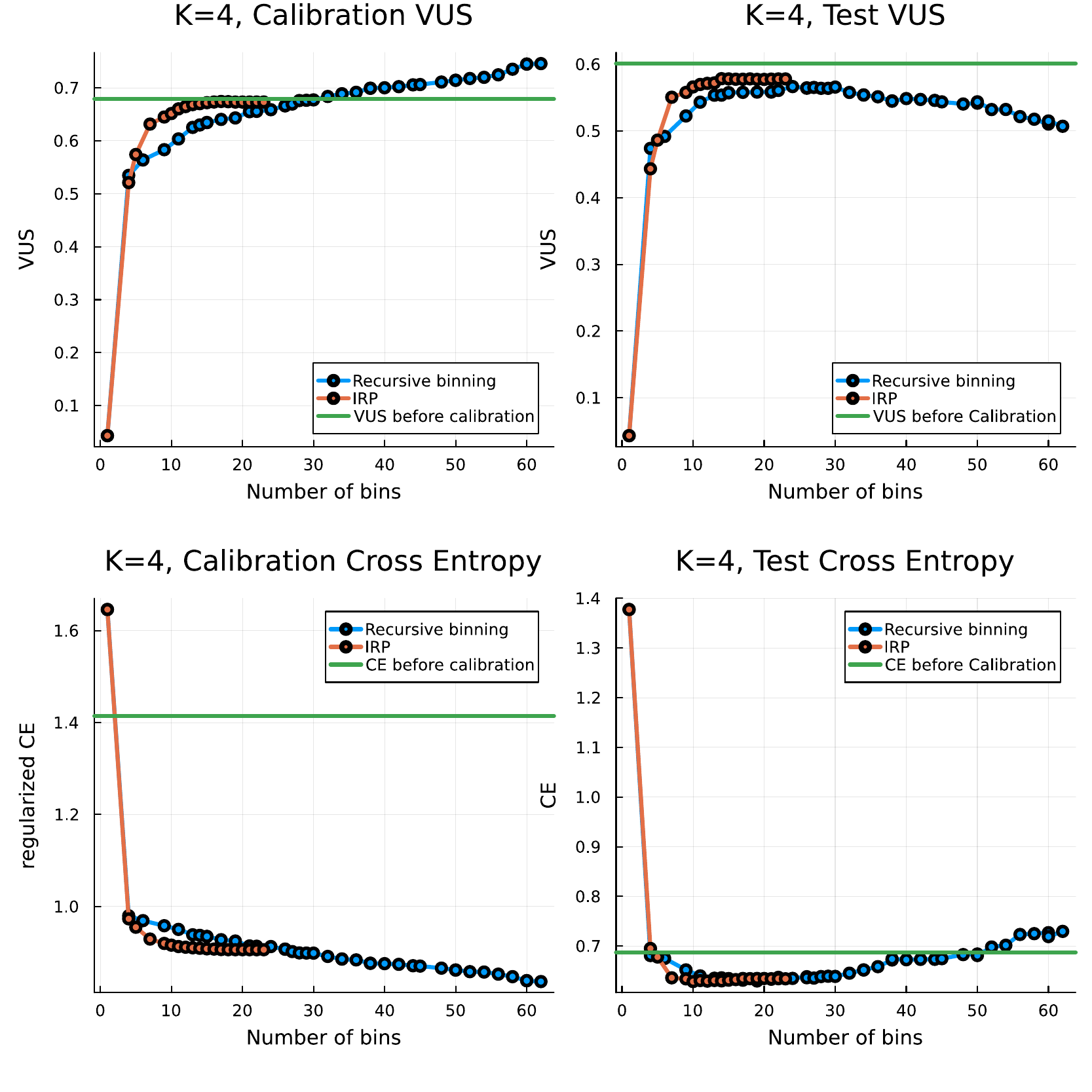}
    \vspace*{-.25cm}
    \caption{For $K=4$, calibration and test cross entropy and VUS, IRP versus nonmonotone recursive binning.}
    \label{fig:4DIRPvsBinning}
\end{figure}

\section*{Acknowledgements}
We   acknowledge support from the French government under the management of the Agence Nationale de la Recherche as part of the ``Investissements d’avenir'' program, reference
ANR-19-P3IA-0001 (PRAIRIE 3IA Institute).

\bibliography{bibliography}

\newpage

\appendix

 \section{Additional figures}
 
Figure \ref{fig:3DCalibrationFunction} illustrates results for the three-class IRP Algorithm \ref{alg:GIRP} on a synthetic dataset presented in the top-left corner of the figure. The non-calibrated predictions are generated by a uniform distribution of points on the three-dimensional simplex. The corresponding labels are chosen to be the argmax of the predictions plus some with noise, the labels are represented on the figure by the color of the dots. We represent the calibration function obtained by setting the color of the points to be the value of the three-dimensional function in RGB (top right corner). On the bottom line, we represent the splits made by our algorithm on the simplex and the resulting regions obtained, with the value of the region corresponding to the mean of the labels on each region, represented again by the RGB color.

\begin{figure}[h]
    \centering
    \includegraphics[scale=0.7]{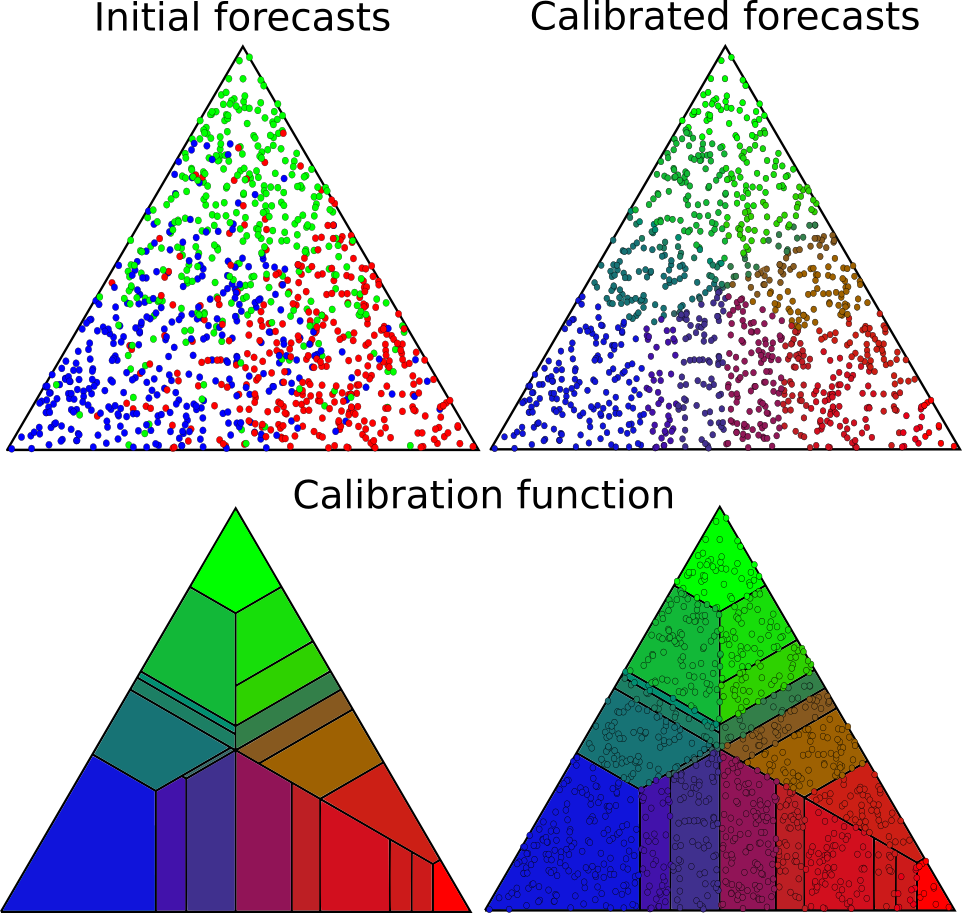}
    \caption{Multi-class IRP on a three-class synthetic calibration set.}
    \label{fig:3DCalibrationFunction}
\end{figure}

Figure \ref{fig:recursiveSplittingROCGraph} displays the resulting three-dimensional ROC surfaces obtained before and after calibration.
\begin{figure}[h]
    \centering
    \includegraphics[scale=0.32]{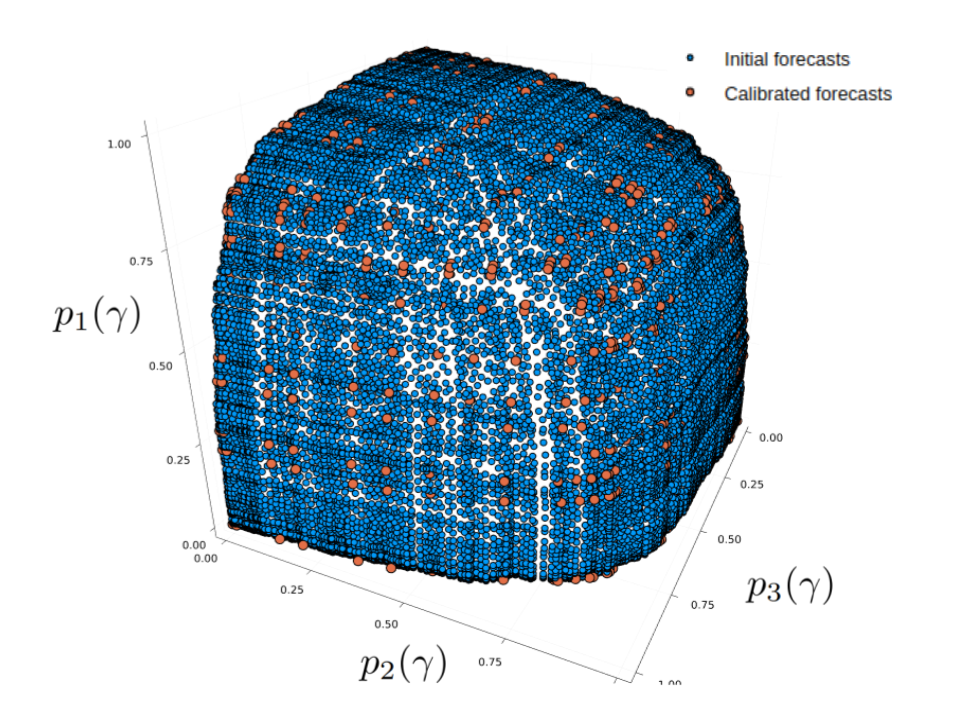}
    \caption{Initial ROC surface (\textbf{blue dots}) and calibrated ROC surface (\textbf{orange dots}) after multi-class IRP on a 3-class synthetic calibration set.}
    \label{fig:recursiveSplittingROCGraph}
\end{figure}

Figure \ref{fig:4DCalibrationFunction} illustrates the result of the four-class IRP Algorithm \ref{alg:GIRP} on the output of a logistic regression classifier trained on the first four classes of the Covertype UCI dataset \citep{covertype}. The four-dimensional simplex is plotted as the regular pyramid in three dimensions.

\begin{figure}[t]
    \centering
    \includegraphics[scale=0.45]{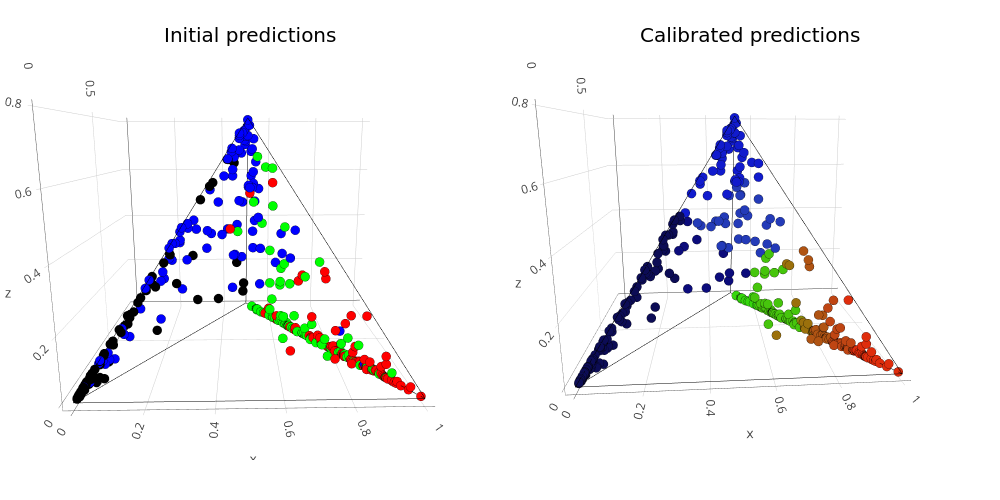}
    \caption{Multi-class IRP on a 4-class calibration set.}
    \label{fig:4DCalibrationFunction}
\end{figure}

  \end{document}